\documentclass[11pt]{article} 

\usepackage{fullpage}

\usepackage{amsmath,amsthm, amssymb}
\usepackage{latexsym}
\usepackage{graphicx}
\usepackage{ifthen}
\usepackage{subcaption}
\usepackage{multicol}
\usepackage{algorithm,algorithmic}
\usepackage[numbers]{natbib}
\usepackage{mathtools}
\usepackage{dsfont}
\usepackage{nicefrac}

\usepackage{wrapfig}
\usepackage{graphicx}

\usepackage{hyperref}
\hypersetup{colorlinks=true,linkcolor=blue,citecolor=blue,urlcolor=blue}

\newtheorem*{theorem*}{Theorem}
\newtheorem*{definition*}{Definition}

\newtheorem{theorem}{Theorem}
\newtheorem{lemma}{Lemma}
\newtheorem{definition}{Definition}

\usepackage{thm-restate}

\renewcommand{\algorithmiccomment}[1]{\bgroup\hfill\footnotesize~#1\egroup}

\usepackage{hyperref}
\hypersetup{colorlinks=true,linkcolor=blue,citecolor=blue,urlcolor=blue}

\newcommand{\conv}{\text{Conv}}
\newcommand{\intt}{\text{int}}
\newcommand{\dom}{\text{dom}}

\DeclareMathOperator{\E}{\mathbb{E}}

\DeclareMathOperator{\poly}{poly}

\newcommand{\bg}{\text{bg}}

\newcommand{\Q}{\mathbb{Q}}                     
\newcommand{\R}{\mathbb{R}}                     
\newcommand{\U}{\mathcal{U}}

\newcommand{\F}{\mathcal{F}}

\renewcommand{\O}{O} 
\renewcommand{\Q}{\mathcal{Q}}
\newcommand{\A}{\mathcal{A}}

\newcommand{\bx}{\mathbf{x}}

\newcommand{\by}{\mathbf{y}}

\newcommand{\bz}{\mathbf{z}}
\renewcommand{\bg}{\mathbf{g}}

\newcommand{\comment}[1]{#1}

\title{Parallelization does not Accelerate Convex Optimization: Adaptivity Lower Bounds for Non-smooth Convex Minimization}

\author{Eric Balkanski\\Harvard University \\ \texttt{ericbalkanski@g.harvard.edu}
\and Yaron Singer\\ Harvard University \\ \texttt{yaron@seas.harvard.edu}\\}

\date{}

\begin{document}

\setcounter{page}{0}

\maketitle
\begin{abstract}
In this paper we study the limitations of parallelization in convex optimization.  A convenient approach to study parallelization is through the prism of \emph{adaptivity} which is an information theoretic measure of the parallel runtime of an algorithm \cite{BS18}.  Informally, adaptivity is the number of sequential rounds an algorithm needs to make when it can execute polynomially-many queries in parallel at every round.  For combinatorial optimization with black-box oracle access, the study of adaptivity has recently led to exponential accelerations in parallel runtime and the natural question is whether dramatic accelerations are achievable for convex optimization. 

\comment{ For the problem of minimizing a non-smooth convex function $f:[0,1]^n\to \mathbb{R}$ over the unit Euclidean ball, we give a tight lower bound that shows that even when $\poly(n)$ queries can be executed in parallel, there is no randomized algorithm with $\tilde{o}(n^{1/3})$ rounds of adaptivity that has convergence rate that is better than those achievable with a one-query-per-round algorithm.
A similar lower bound was obtained by Nemirovski~\cite{nemirovski1994parallel}, however that result  holds for the $\ell_{\infty}$-setting instead of $\ell_2$. In addition, we also show a  tight lower bound that holds for Lipschitz and strongly convex functions.}

At the time of writing this manuscript we were not aware of Nemirovski's result.  The construction we use is similar to the one in \cite{nemirovski1994parallel}, though our analysis is different.  Due to the close relationship between this work and \cite{nemirovski1994parallel}, we view the research contribution of this manuscript limited and it should serve as an instructful approach to understanding lower bounds for parallel optimization. 
\end{abstract}

\newpage


\section{Introduction}

In this paper we study the limitations of parallelization in convex optimization.  Since applications of convex optimization are ubiquitous across machine learning and data sets become larger, there is consistent demand for accelerating convex optimization.  For over 40 years computer science has formally studied acceleration of computation via parallelization~\cite{F78,G78,S79}.  Our goal in this paper is to study whether parallelization can generally accelerate convex optimization.

A convenient approach to study parallelization is through the prism of \emph{adaptivity}.   Adaptivity is an information theoretic measure of the parallel runtime of an algorithm, which in many cases also translates to a computational upper bound, up to lower order terms.  Informally, adaptivity is the number of sequential rounds of an algorithm where every round allows for polynomially-many parallel queries.  Adaptivity is studied across a wide variety of areas, including sorting \cite{Val75, Col88, BMW16},  communication complexity \cite{papadimitriou1984communication, duris1984lower, nisan1991rounds},  multi-armed bandits \cite{AAAK17}, sparse recovery \cite{HNC09,IPW11, haupt2009compressive}, and property testing \cite{CG17, chen2017settling}.  In the celebrated PRAM model adaptivity is the \emph{depth} of the computation tree.  More generally, in any parallel computation model, adaptivity lower bounds the runtime of algorithms that make polynomially-many parallel queries.


For combinatorial optimization, the study of adaptivity has recently led to dramatic accelerations in parallel runtime.  For the canonical problem of maximizing a monotone submodular function under a cardinality constraint, a recent line of work initiated in~\cite{BS18} introduces techniques that produce exponential speedups in parallel runtime.  Until very recently the best known adaptivity (and hence best parallel running time) for obtaining constant factor approximations for a submodular function $f : \{0,1\}^n \rightarrow \R$ was linear in $n$. In contrast,~\cite{BS18} and the line of work that follows~\cite{BS18b,BBS18,EN18,BRS18,fahrbach2018submodular,chekuri2018submodular} achieve constant, and even optimal, approximation guarantees in $\O(\log n)$ adaptive steps.     
%

%
For convex minimization, in some special cases, parallelization provides non-trivial speedups.  An important example that is well studied in machine learning is when the objective function is decomposable, i.e. when $f(\mathbf{x}) = \sum_{i}f_i(\mathbf{x})$.  In this case parallelism allows computing stochastic subgradients of the convex functions $\{f_{i}\}_i$ simultaneously during iterations of stochastic gradient descent~\cite{DGSX12,DBW12,RRWN11}.  Another special case is when the function is low dimensional.  Recently Duchi et al. show that for $f:[0,1]^n\to\mathbb{R}$ when the number of queries can be \emph{exponential} in the dimension $n$, then parallelization can accelerate minimization when $f$ is either Lipschitz convex or strongly convex and strongly smooth~\cite{duchi2018minimax}.  The natural question is whether algorithms that can execute $\poly(n)$ function evaluations in each iteration can achieve faster convergence rates than those that make a single evaluation in every iteration.
\begin{center}
\emph{
Can parallelization accelerate convex optimization?}
\end{center} 
 
\subsection{Main result}

Our main result is a spoiler.  We show that, in general, parallelization does not accelerate convex optimization.    
In particular,  for the problem of minimizing a Lipschitz and strongly convex function $f:[0,1]^n\to \mathbb{R}$ \comment{over the unit Euclidean ball}, we give a tight lower bound that shows that even when $\poly(n)$ queries can be executed in parallel, there is no randomized algorithm that has convergence rate that is better than those achievable with a one-query-per-round algorithm~\cite{nesterov2013introductory,shamir2013stochastic}.    

\begin{restatable}{rThm}{thmmaintwo}
\label{thm:main2}  
For any $G, D > 0$ and $r \in [n]$, there exists a family of convex functions $\F_r$ with $\|\bg\|^2 \leq G^2$ for all subgradients $\bg$ of all $f \in \F_r$, such that for any $r$-adaptive algorithm $\A$, there exists $f \in \F_r$ for which $\A$ returns $\bx_r \in \mathcal{W}$ such that
$$f(\bx_r) - \min_{\bx \in \mathcal{W}}f(\bx) \geq GD \left(\frac{1}{2\sqrt{r+1}} - \frac{(r+1/2)\log n}{\sqrt{n}}\right)$$
with  probability $\omega(1/n)$ over the randomization of $\A$ and
with  domain $\mathcal{W} = \left[-\frac{D}{2\sqrt{n}}, \frac{D}{2\sqrt{n}}\right]^n$ of diameter $\max_{\bx, \bx' \in \mathcal{W}} \|\bx - \bx'\|\comment{_2} = D$. In particular, for any $r \in o(n^{1/3} / \log n)$, 
$$f(\bx_r) - \min_{\bx \in \mathcal{W}}f(\bx) \geq GD \cdot \left(1 - o(1)\right)\cdot\frac{1 }{2\sqrt{r+1}}.$$
\end{restatable}

This  $\Omega(1/\sqrt{r})$ convergence rate matches (up to lower order terms) the  convergence rate  of standard sequential algorithms for Lipschitz convex functions~\cite{nesterov2013introductory,shamir2013stochastic}. \comment{A similar result by Nemirovski \cite{nemirovski1994parallel} implies a bound with the same convergence rate, but for the $\ell_\infty$ setting instead of $\ell_2$. We also obtain the following convergence rate for Lipschitz and strongly convex functions.}

\begin{restatable}{rThm}{thmmainone}
\label{thm:main1}  
For any $\lambda, G > 0$ and $r \in [n]$, there exists a family of $\lambda$-strongly convex functions $\F_r$, with $\|\bg\|^2 \leq G^2$ for all subgradients $\bg$ of all $f \in \F_r$ over domain $\mathcal{W}$, such that for any $r$-adaptive algorithm $\A$, there exists $f \in \F_r$ for which $\A$ returns $\bx_r \in \mathcal{W}$ such that 
$$f(\bx_r) - \min_{\bx \in \mathcal{W}}f(\bx)\geq \frac{G^2}{\lambda}\left(\frac{1 }{8(r+1)} - \sqrt{\frac{r+ 1}{n}}\frac{\log n}{2}\right)$$ with  probability $\omega(1/n)$ over the randomization of $\A$ and with the box $\left[-\frac{G}{2\lambda\sqrt{n(r+1)}}, \frac{G}{2\lambda\sqrt{n(r+1)}}\right]^n$ as domain $\mathcal{W}$. In particular, for any $r \in o(n^{1/3} / \log n)$, 
$$f(\bx_r) - \min_{\bx \in \mathcal{W}}f(\bx) \geq \frac{G^2}{\lambda} \cdot (1-o(1)) \cdot \left(\frac{1 }{8(r+1)}\right).$$
\end{restatable}

Again, this $\Omega(1/r)$ convergence rate matches (up to lower order terms) the convergence rate of standard sequential algorithms with  one query per round for $\lambda$-strongly convex functions~\cite{nesterov2013introductory,shamir2013stochastic}.

\paragraph{Some remarks.}  
The lower bounds hold for both deterministic and randomized algorithms for optimizing non-stochastic functions $f$.  The lower bounds thus trivially hold for the stochastic case as well, since it is strictly harder.  Similarly, these lower bounds also hold for decomposable convex functions since a decomposable function composed of a single function is a special case.   
The lower bounds hold with high probability over the randomization of the algorithm, and trivially also hold in expectation with an additional $1- o(1)$ multiplicative term.   Finally, the lower bounds hold for both zeroth and first-order oracles. We present these lower bounds for zeroth-order oracles, which extend to first-order oracles since  first-order oracles can be obtained from zeroth-order oracles when poly-many queries are allowed per round by querying a small ball around the point of interest in a single round. 

\subsection{Technical overview}

We introduce a new framework to argue about the information theoretic limitations of algorithms when given access to polynomially-many queries in each round. \comment{Our construction is similar to the construction from \cite{nemirovski1994parallel}.}

We begin by describing a simple class of functions that can be used to show the lower bound.  We then reduce the problem to showing that the class of functions we define respects two conditions: \emph{indistinguishability} and \emph{gap}.  The main technical challenge is in proving that the family of functions we construct satisfies these indistinguishability and gap conditions.


Satisfying these two conditions requires finding, for any algorithm $\A$, two functions $f_\by$ and $f_\bz$ in the family of functions which have equal value over queries by $\A$ but different optima. The main difficulty is that we pick  $f_\by$ and $f_\bz$ depending on the queries of the algorithm $\A$, but $\A$ can learn partial information about $f_\by$ and $f_\bz$ from those queries. Thus, the queries of the algorithm are dependent on $f_\by$ and $f_\bz$, which creates a cycle of dependence. 

The main conceptual part of the analysis is in finding such $f_\by$ and $f_\bz$. We do so using an oracle $f_\A$ which is defined \emph{adaptively}  and over multiple rounds as it receives queries from an algorithm $\A$. We call such an oracle which is dependent on an algorithm $\A$ an \emph{obfuscating oracle}. This construction contains multiple subtleties due to the complex dependencies between $f_\A$ and $\A$.  Showing the existence of an obfuscating oracle with the desired properties is also non-trivial. It requires a probabilistic argument that derandomizes an algorithm by showing that it is sufficient to argue about properties of  a deterministic query by an algorithm to a random function instead of random queries to a deterministic function.


\subsection{Related work}

 The study of the hardness of convex optimization was initiated in the seminal work of \cite{nemirovskii1983problem} which introduced the standard model for lower bounds in convex optimization (see 
\cite{nesterov2013introductory, bubeck2015convex} for a simplified presentation). In that model, there is a black-box oracle for a convex function $f$ such that the algorithm queries points $\bx$ and receives answers $f(\bx)$ from the oracle. 
There is a rich line of work on information theoretic lower bounds for the number of sequential queries needed for convex optimization in the setting where the oracle $f$ is stochastic, e.g. \cite{raginsky2009information,raginsky2011information}, or non-stochastic, e.g. \cite{agarwal2009information,woodworth2016tight,braun2017lower}. In this paper, we  consider the basic case where the oracle is not stochastic and note that any lower bound in the non-stochastic setting trivially extends to the stochastic setting.
Since the standard model for lower bounds in convex optimization uses a black-box oracle access setting, adaptivity is well-suited for the study of lower bounds for parallel convex optimization.


 
\comment{As previously mentioned, a similar result by Nemirovski \cite{nemirovski1994parallel} implies a bound with the same convergence rate, but for the $\ell_\infty$ setting instead of $\ell_2$.}
Very recent work has also obtained lower bounds on the convergence rates of adaptive algorithms for convex optimization \cite{smith2017interaction, woodworth2018graph, duchi2018minimax}. The exact settings vary, but the high level goal is the same as ours, which is to derive convergence rates for algorithms which allow multiple parallel queries in every round. We give lower bounds which improve over these previous lower bounds. In particular, the convergence rates in \cite{duchi2018minimax}  exponentially decrease in the dimension $n$. The lower bounds in \cite{woodworth2018graph} have  a $1/\sqrt{m}$ dependence term where $m$ is the number of queries, while our lower bounds are independent of $m$ and we only assume that the number of queries is at most $\poly(n)$. Motivated by applications to local differential privacy, \cite{smith2017interaction} obtained lower bounds on the convergence rate that have an exponential dependence on the number of rounds $r$, while we obtain the optimal $1/\sqrt{r}$ and $1/r$ rates. Also related to adaptivity for convex optimization is the work of \cite{perchet2016batched}, which studies adaptivity in a bandit setting and obtains regret bounds for strategies that can be updated only a limited number of times.

 

Non-adaptivity, i.e. $1$-adaptive algorithms, has been studied for convex optimization in~\cite{BS17a}, where it has been shown that there is no algorithm that can obtain even a constant error using fewer than exponentially-many (in the dimension $n$) samples drawn from any distribution. This hardness result for non-adaptive algorithms for convex optimization motivated our study of algorithms with $r > 1$ rounds of adaptivity.   More generally, non-adaptive algorithms have also been studied for combinatorial optimization to study the power and limitations of algorithms whose input is learned from observational data~\cite{balkanski2017limitations,BS17,BIS17,BRS16,RBGS18}.


%
%
%
%

\subsection{Adaptivity} 
The \emph{adaptivity} of an algorithm is the number of sequential rounds of queries it makes, where every round allows for $\poly(n)$ parallel queries where $n$ is the dimension of the problem.

\begin{definition*} Given an oracle $f$, an algorithm  is \textbf{$r$-adaptive} if every query $\bx$ to the oracle  occurs at a \emph{round} $i \in [r]$ such that $\bx$ is independent of the answers $f(\by)$ to all other queries $\by$ at round $i$.
\end{definition*}

We note that the definition is stated for zero-order oracles (given $\mathbf{x}$ the oracle returns $f(\mathbf{x})$), but as previously mentioned, we emphasize that it is equivalent to assuming first-order oracle access since $\poly(n)$ queries to $f$ are allowed in every round, and thus  subgradients can be obtained in one round of querying $f$.

\subsection{Paper organization}

In Section~\ref{sec:construction}, we construct the family of Lipschitz convex functions that is hard to optimize in $r$ adaptive rounds of queries and present two simple sufficient conditions, called the indistinguishability and gap conditions, on a class of functions to obtain the lower bound. In Section~\ref{sec:obfuscating}, we present the obfuscating oracle, which is used to find two functions in the hard family of functions that satisfy these two conditions. We show that these two functions satisfy the indistinguishability and gap conditions in Section~\ref{sec:analysis}. Finally, in Section~\ref{sec:stronglyconvex}, we extend the construction and the lower bound to $\lambda$-strongly convex functions.


\section{The Construction of the Hard Family of Functions}
\label{sec:construction}

In this section, we construct the family of functions that cannot be optimized in $r$ rounds of queries. We then describe two simple conditions that together are sufficient for showing the hardness of optimizing a class of functions in $r$ rounds.

\subsection{The hard family of functions}
\label{sec:family}

We give the construction of the family of Lipschitz convex functions $\F_r$    for the lower bound for $r$-adaptive algorithms. In Section~\ref{sec:stronglyconvex}, we extend this construction to obtain a family of functions which is Lipschitz and $\lambda$-strongly convex. The  functions $f_\by \in \F_r$ are  parameterized by a binary vector $\by \in \{-1,  1\}^{n}$ and optimized over domain $\mathcal{W} = \left[-D/(2\sqrt{n}), D/(2\sqrt{n})\right]^n$, which is the box of diameter $D$. For a vector $\bx \in \mathcal{W}$ (and similarly for $\by \in \{-1,  1\}^{n}$), we often break $\bx$ into $r + 1$ blocks  $\bx_1, \ldots, \bx_{r+1}$ of $n/(r+1)$ consecutive entries of $\bx$, where  $$\bx_i := \bx\left[(i-1)\frac{n}{r+1}+1: i\frac{n}{r+1}\right] \in \R^{n/(r+1)}.$$ The functions are in terms of some $\gamma > 0$ which we later define.  Formally, the function $f_{\by}$ is defined  as 
$$f_{\by}(\bx) := \gamma \cdot \max_{i \in [r+1]} \left(\bx^{\intercal}_i \by_i -  2 i  \epsilon\right)$$
 where $ \epsilon := \frac{D \log n}{2\sqrt{r+1}}$.  The family of functions for which we show a lower bound for $r$-adaptive algorithms is 
$$\F_r := \{f_{\by} : \by \in \{-1,  1\}^{n}\}.$$

We discuss some informal intuition for the hardness of optimizing $\F_r$ in $r$-adaptive rounds before giving the formal argument. 
The main idea behind these functions is  that an algorithm needs to learn all $\by_i$, $i \in [r+1]$, to optimize $f_{\by}$ within good accuracy,  but that it cannot learn $\by_i$ before round $i$. The reason is that for a query $\bx$ by an algorithm at any round $j < i$, if $\bx^{\intercal}_j \by_j$ and  $\bx^{\intercal}_i \by_i$ concentrate, i.e., $|\bx^{\intercal}_j \by_j| < \epsilon$ and $|\bx^{\intercal}_i \by_i | < \epsilon$, then $$\bx^{\intercal}_j \by_j -  2 j  \epsilon > \bx^{\intercal}_i \by_i -  2 i  \epsilon.$$ Note that by the definition of $f_\by$, conditioned on  $\bx^{\intercal}_j \by_j -  2 j  \epsilon > \bx^{\intercal}_i \by_i -  2 i  \epsilon$, the value of $f_\by$ is independent of $\by_i$ and the algorithm does not learn $\by_i$. Informally, if an algorithm has not yet learn $\by_i$ at some round $j$, $\bx^{\intercal}_i \by_i$ is likely to concentrate for the queries $\bx$ by this algorithm at round $j$.


Observe that a minimizer for $f_{\by}$ over $\mathcal{W}$ is 
$\bx^{\star}$ such that 
$$\bx^{\star}_j = \begin{cases} \frac{D}{2 \sqrt{n}} & \text{ if }y_j = -1 \\
- \frac{D}{ 2 \sqrt{n}} & \text{ if }y_j = 1 \end{cases}$$ 

 If an algorithm cannot learn $\by_{r+1}$ in $r$-adaptive rounds, then $\bx^{\intercal}_{r+1} \by_{r+1}$ is likely to concentrate for the solution $\bx$ returned by the algorithm. If $\bx^{\intercal}_{r+1} \by_{r+1}$ concentrates, then $\bx$ is a bad solution compared to $\bx^{\star}$.

\subsection{Two sufficient conditions for hardness}

We reduce the analysis of the lower bound for $\F_r$ to showing that for any algorithm $\A$, there exists $f_\by, f_\bz \in \F_r$ that satisfy two simple conditions.  Informally, the first condition, called indistinguishability, states that the functions $f_\by$ and $f_\bz$ have, with high probability, equal value over all queries of the algorithm $\A$. On the other hand, the second condition, called $\alpha$-gap, states that there is no solution which simultaneously $\alpha$-approximates the optimal solutions of both functions $f_\by$ and $f_\bz$. It is easy to show that if a class of functions $\F$ contains two such functions $f_\by$ and $f_\bz$ for any algorithm $\A$, then $\F$ is hard to optimize since $f_\by$ and $f_\bz$ need to be distinguished for an algorithm $\A$ to have good performance. 

\begin{theorem}
\label{thm:meta}
 Let $\A$ be some algorithm for a class of functions $\F$. Assume there exists $f_\by, f_\bz \in \F$ with the properties of
\begin{itemize}
\item \textbf{Indistinguishability:} for all rounds $i \leq r$, let $\Q_i$ be the queries at round $i$ by $\A$, which are adaptive to the answers $f_\by(\bx)$ by $f_\by$ to queries $\bx$ by $\A$ to $f_\by$ at rounds $j < i$. Then, with probability $1 - n^{-\omega(1)}$ over $\A$, for all $\bx \in \Q_i$,    $$f_\by(\bx) =  f_\bz(\bx).$$
\item \textbf{Gap:} Minimizers for $f_\by$ and $f_\bz$ have equal value, i.e., 
$$\min_{\bx \in \mathcal{W}} f_\by(\bx) = \min_{\bx \in \mathcal{W}} f_\bz(\bx),$$ but for all $\bx \in \mathcal{W}$:
$$\max(f_\by(\bx), f_\bz(\bx))  - \min_{\bx \in \mathcal{W}} f_\by(\bx) > \alpha.$$
\end{itemize}
Then,  there is no $r$-adaptive algorithm that finds for all $f \in \F$, with probability strictly larger than $\omega(1/n)$ over the randomization of the algorithm, a solution $\bx^r$ s.t. $f(\bx^r) - \min_{\bx \in \mathcal{W}} f(\bx) \leq \alpha$.
\end{theorem}
\begin{proof}
Consider an algorithm $\A$ for $\F$. Let $f_\by, f_\bz \in \F$ be the functions satisfying the indistinguishability and $\alpha$-gap conditions.

 Pick the function oracle to be either $f_\by$ or $f_\bz$ with probability $1/2$ each. By the indistinguishability property, with probability $1 - n^{-\omega(1)}$ over the queries $\A$, the answers of the oracle to all queries by $\A$ are independent of whether the oracle is for $f_\by$ or $f_\bz$.  Thus, the decisions by the algorithm are independent of the randomization over $f_\by$ and $f_\bz$. By the gap condition,  with probability $1 - n^{-\omega(1)}$ over the algorithm, we conclude that  the algorithm returns a (possibly randomized) $\bx'$ such that either $\E_{\bx'}[f_\by(\bx')] - \min_{\bx \in \mathcal{W}} f_\by(\bx) > \alpha$ or $\E_{\bx'}[f_\bz(\bx')] - \min_{\bx \in \mathcal{W}} f_\bz(\bx) > \alpha$.
\end{proof}


\section{The Obfuscating Oracle}
\label{sec:obfuscating}

In this section, we construct two functions which satisfy the indistinguishability and gap conditions for Theorem~\ref{thm:meta}. This construction relies on a tool  called an obfuscating oracle. The definition and construction of an obfuscating oracle for $\F_r$ is the main conceptual part of the analysis. Recall that to obtain the two desired conditions, we need to show that for any $r$-adaptive algorithm $\A$, there exist two functions $f_\by, f_\bz \in \F_r$ that have equal value over all queries by $\A$ but do not have a common minimizer.

 We start with a high level overview of the structure of this pair of functions $f_{\by}, f_{\bz} \in \F_r$. Recall that a function in $\F_r$ is defined by a binary vector $\by \in \{-1,1\}^n$ broken into $r+1$ vectors $\by_1, \ldots, \by_{r+1}$. For our construction of $f_{\by}$ and $f_{\bz}$, $\by_i = \bz_i$ for $i \leq r$  but $\by_{r+1} \neq \bz_{r+1}$. The identical first $r$ blocks imply the indistinguishability condition and the different last block  implies the gap condition. More precisely, we wish to pick $\by_1, \ldots, \by_{r}, \by_{r+1}$ such that for all queries $\bx$ at round $i$, $|\bx^{\intercal}_j \by_j| < \epsilon$ for all $j \geq i$. Note that by the definition of $f$, this implies that $f_\by(\bx) = f_\bz(\bx) = \gamma \cdot \max_{\ell \in [i]} \left(\bx^{\intercal}_\ell \by_\ell -  2 \ell  \epsilon\right)$ and thus indistinguishability. Intuitively, the consequence is that  algorithm $\A$ does not learn $\by_i$ before round $i$, and in particular does not distinguish $\by_{r+1}$ and $\bz_{r+1}$ at the end of the $r$ rounds of the algorithm.

 An important subtlety which complicates the analysis is that algorithm $\A$ can learn \emph{some} information about $\by_i$ before round $i$. This is since with query $\bx$ at round $j < i$, with the answer  $f_\by(\bx) = f_\bz(\bx) = \gamma \cdot \max_{\ell \in [j]} \left(\bx^{\intercal}_\ell \by_\ell -  2 \ell  \epsilon\right)$ of the oracle, $\A$ learns that $\bx^{\intercal}_i \by_i -  2 i  \epsilon
 <  \max_{\ell \in [j]} \left(\bx^{\intercal}_\ell \by_\ell -  2 \ell  \epsilon\right)$. Thus, we cannot argue that $\A$ does not learn any information about $\by_i$ and  that queries at round $j < i$ are completely independent of $\by_i$. The remaining of this section is devoted to finding $\by_1, \ldots, \by_{r+1}$ such that $f_{\by}$ and $f_{\bz}$, where $\bz_i = \by_i$ for $i \leq r$ and $\bz_{r+1} = - \by_{r+1}$, satisfy the indistinguishability and gap conditions. The main difficulty is that we wish to pick $\by_1, \ldots, \by_{r+1}$ depending on the queries of the algorithm $\A$, but since $\A$ is adaptive and it can learn partial information about $\by_1, \ldots,  \by_{r+1}$, the queries of the algorithm are dependent on $\by_1, \ldots, \by_{r+1}$, which creates a cycle of dependence. Some subtle dependencies between $\by_1, \ldots, \by_{r+1}$ and $\A$ are needed.

\subsection{The definition of an obfuscating oracle}

Instead of a function oracle $f$ which is defined before the algorithm $\A$ starts querying $f$, an obfuscating oracle is an oracle which is adaptively defined as it interacts with the queries of an algorithm. In particular, the answers of an obfuscating oracle might be dependent on the  previous queries by $\A$ and on the round $i$ in which a query occurs, which is of course not possible for a function oracle. 

In our case, we construct an obfuscating oracle $f_\A$ which, similarly as  function oracles in $\F_r$, depends on points $\by_1, \ldots, \by_{r+1}$. The main idea is to define point $\by_i$ for obfuscating oracle $f_\A$  depending on the queries of $\A$ at rounds $j \leq i$, as illustrated in Figure~\ref{fig:resistantoracle}. Deferring the choice of $\by_i$ for $f_\A$ until round $i$ of $\A$ is the key part of the obfuscating oracle which allows us to argue about indistinguishability, and involves important subtleties which are discussed in Section~\ref{sec:resistantconstruction} where we formally construct the obfuscating oracle.  First, we formally defining obfuscating oracles. An obfuscating oracle $f_\A$ to $\A$ is an oracle that is defined by its interactions with $\A$.

\begin{figure}
\centering
\includegraphics[scale=0.6]{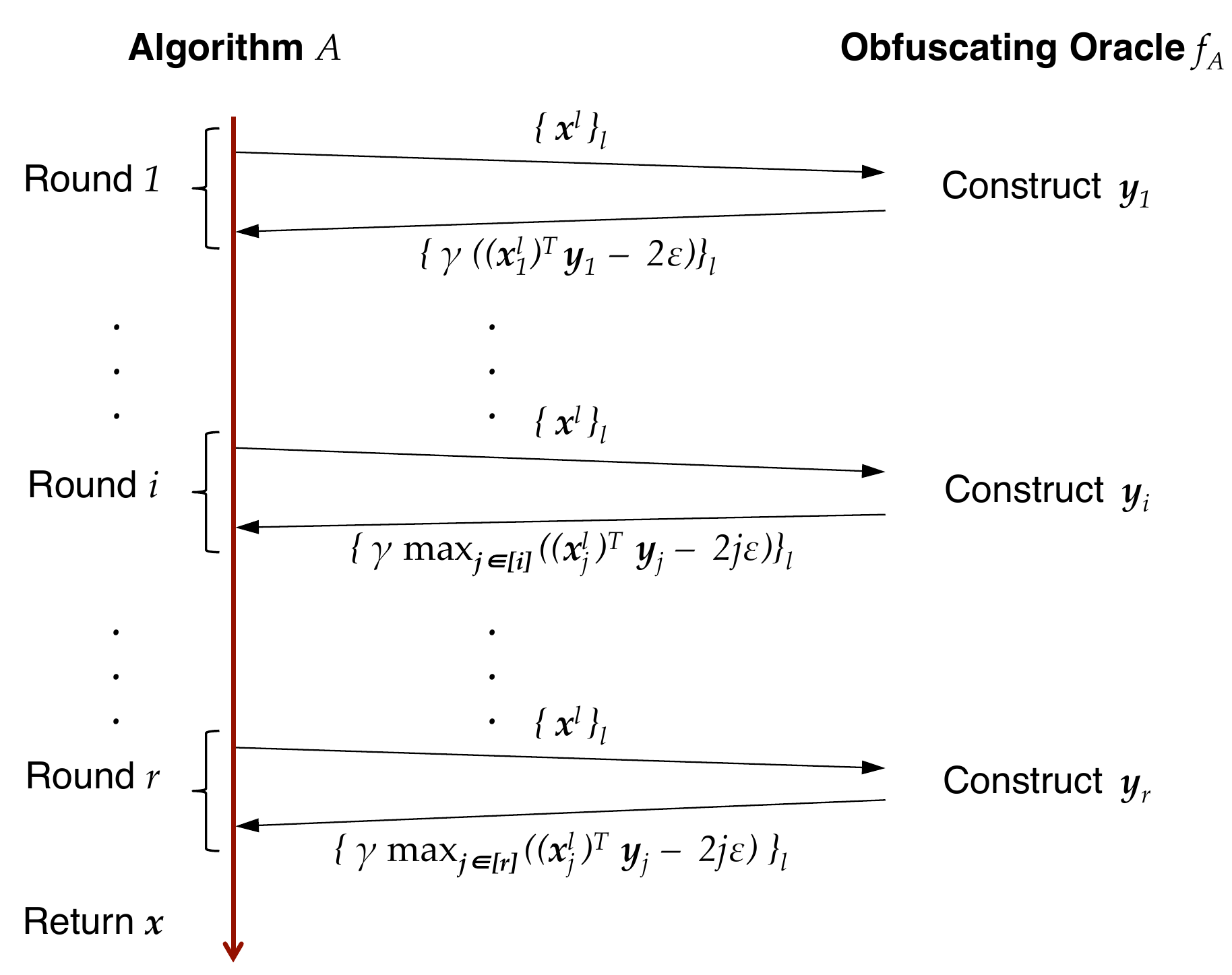}
\caption{The interactions between the obfuscating oracle $f_\A$ and the algorithm $\A$}
\label{fig:resistantoracle}
\end{figure}

\begin{definition}
\label{def:resisting}
 Let $\A$ be some algorithm. An obfuscating oracle $f_\A$ is defined inductively on the $i$th round of queries by $\A$ to $f_\A$. At round $0$, $f_\A(\bx)$ is undefined for all $\bx$. At round $i$, we assume that $f_\A(\by)$ is defined for all queries $\by$ at rounds $j < i$. We consider queries $\mathcal{Q}_i$ at round $i$ by algorithm $\A$, which has  received answers $f_\A(\by)$ for all queries $\by$ at rounds $j < i$ from the oracle. For all $\bx \in \mathcal{Q}_i$, $f_\A(\bx)$ is defined independently of all queries at rounds $j > i$.
\end{definition}

Next, we formalize an obfuscating condition  which implies the indistinguishability property. The obfuscating condition states that, with high probability, there exist two functions in $\F_r$ which have equal value with $f_\A$ over all queries by $\A$.

\begin{lemma}
\label{lem:obfuscating}
Assume that for all $r$-adaptive algorithms $\A$, there exists, with probability $1 - n^{-\omega(1)}$, an obfuscating oracle $f_\A$ such that, for some $f_\by, f_\bz$,  
\begin{itemize}
\item \textbf{Obfuscating condition:} for all queries $\bx$ by $\A$ to obfuscating oracle $f_\A$, $$f_\A(\bx) = f_\by(\bx) = f_\bz(\bx).$$
\end{itemize}
 Then $f_\by$ and $f_\bz$ satisfy the indistinguishability property.
\end{lemma}
\begin{proof}
Assume that the obfuscating condition holds, which occurs with probability $1 - n^{-\omega(1)}$. We show that $f_\by$ and $f_\bz$ satisfy the indistinguishability property by induction on the round $i$. Consider round $i$ and assume that for all queries $\bx$ by $\A$ to function oracle $f_\by$ from previous rounds, $f_\A(\bx) = f_\by(\bx) = f_\bz(\bx)$. Since $f_\A$ and $f_\by$ have equal value over all previous queries, the algorithm cannot distinguish between if it is querying $f_\A$ or $f_\by$. Thus, the queries $\Q_i$ by $\A$ at round $i$ to function oracle $f_\by$ are identical to the queries by $\A$ at round $i$ to obfuscating oracle $f_\A$. By the obfuscating condition, for all $\bx \in \Q_i$, we have that   $f_\A(\bx) = f_\by(\bx) = f_\bz(\bx)$. Since this holds with probability $1 - n^{-\omega(1)}$ for all rounds $i \leq r$, we get that $f_\by$ and $f_\bz$ satisfy the indistinguishability property.
\end{proof}

\subsection{The construction of the obfuscating oracle}
\label{sec:resistantconstruction}

We construct an obfuscating oracle $f_\A$ for $\F_r$.  Let $\A$ be any $r$-adaptive algorithm. We construct $f_\A$ inductively on the round  $i$ of queries. At round $i$, let $\Q_i$ be  the (possibly randomized) collection of queries by $\A$, after having received answers $f_\A(\bx)$ from $f_\A$ for queries $\bx$ at rounds $j < i$.  Let $\by_1, \ldots, \by_{i-1}$ be the points previously constructed by $f_\A$. Then,
\begin{itemize}
\item Let  $\by_i \in \{-1,1\}^{n/(r+1)}$ be a binary vector such that
$$\Pr_{\bx \in \cup_{j=1}^i \Q_j}\left[|\bx^\intercal_i \by_i| < \epsilon, \text{ for all } \bx \in \cup_{j=1}^i \Q_j\right] = 1 - n^{-\omega(1)},$$
i.e.,  for all queries $\bx \in \cup_{j=1}^i \Q_j$, $\bx^\intercal_i \by_i$ concentrates with high probability. 

\item The obfuscating oracle $f_\A$ at round $i$  answers, for all $\bx \in \Q_i$,
$$f_\A(\bx) = \gamma \cdot \max_{j \in [i]} \left(\bx^{\intercal}_j \by_j -  2 j  \epsilon \right).$$
\end{itemize}
 Finally, let $\by_{r+1} \in \{-1,1\}^{n/(r+1)}$ be a binary vector such that $$\Pr_{\bx \in \cup_{j=1}^r \Q_r}\left[|\bx^\intercal_{r+1} \by_{r+1}| < \epsilon, \text{ for all } \bx \in \cup_{j=1}^r \Q_r\right] = 1 - n^{-\omega(1)}.$$
 
 Thanks to the obfuscating oracle, we are now ready to define $\by$ and $\bz$ for the two functions $f_\by$ and $f_{\bz}$ for which we show the indistinguishability and gap conditions. The function $f_\by$ is defined by the $r+1$ blocks $\by_i$ constructed by the obfuscating oracle $f_\A$ and $f_{\bz}$ is defined by the $r+1$ blocks $\bz_i$ such that $\bz_i := \by_i$ for $i \leq r$ and $\bz_{r+1} = - \by_{r+1}$.

The crucial part of $f_\A$ is that the maximum is over $j \in [i]$ instead of $j \in [r+1]$. There are multiple subtleties and difficulties with the above construction of the obfuscating oracle, which we discuss next. 
\begin{itemize}
\item The definition of  $\by_i$ at round $i$ must be carefully constructed to not contradict an answer of $f_\A$ to a query $\bx$ from a round $j < i$. In other words, since $f_\A$ answered  $f_\A(\bx) = \gamma \cdot \max_{\ell \in [j]} \left(\bx^{\intercal}_\ell \by_\ell -  2 \ell  \epsilon \right)$ to a query $\bx$ at round $j < i$, we wish to have $\by_i$ such that $$\gamma \cdot \max_{\ell \in [i]} \left(\bx^{\intercal}_\ell \by_\ell -  2 \ell  \epsilon \right) = \gamma \cdot \max_{\ell \in [j]} \left(\bx^{\intercal}_\ell \by_\ell -  2 \ell  \epsilon \right)$$ for the obfuscating condition. It is for this reason that $\by_i$ is defined so that the concentration of $\bx^\intercal_i \by_i$ not only holds for queries $\Q_i$ at round $i$, but for all queries $\cup_{j=1}^i \Q_j$ at rounds $j \leq i$.
\item The obfuscating oracle $f_\A$ does not always construct $\by_1, \ldots, \by_{r+1}$ such that $f_\A(\bx) = f_{\by}(\bx)$ for  $f_{\by} \in \F_r$ for all queries $\bx$. This is because the concentration property of $\by_i$ only holds with high probability over the queries of the algorithm. Thus,  for a query $\bx$ at round $i$ answered with $\gamma \cdot \max_{\ell \in [o]} \left(\bx^{\intercal}_\ell \by_\ell -  2 \ell  \epsilon \right)$ by $f_\A$, $f_{\by}(\bx) = \gamma \cdot \max_{\ell \in [i]} \left(\bx^{\intercal}_\ell \by_\ell -  2 \ell  \epsilon \right)$ only holds with high probability and the answers of $f_\A$ might not correspond to a function in $\F_r$. This even implies that for a same queries $\bx$ at two different rounds, $f_\A$ might answer differently.
\item Note that for all $i$, the queries $\Q_i$ at round $i$ are not independent of $\by_{i-1}$ and $\by_i$ is not independent of $\Q_i$. Thus, there are multiple layers of dependencies between $\A$ and $f_\A$.
\item Finally, and most importantly, it is not trivial that there exists $\by_i$ satisfying the concentration condition for the definition of $f_\A$ at round $i$. Showing that for any randomized $\poly(n)$ queries at rounds $j \leq i$ by $\A$, there exists $\by_i$ such that, with high probability, for all these queries $\bx$, $\bx^\intercal_i \by_i$ concentrates is an important part of the analysis which is shown in Lemma~\ref{lem:existence}.
\end{itemize}

We use  the term $1 - n^{-\omega(1)}$, and more precisely $1 - 1/n^{\log n}$ for our purposes,  to apply union bounds over at most $\poly(n)$ events each happening with probability $1 - n^{-\omega(1)}$. This is useful since the number of queries is at most $\poly(n)$.





\section{Proof of the Main Theorem}
\label{sec:analysis}

In this section, we show that the functions $f_\by$ and $f_\bz$ constructed in the previous section satisfy the indistinguishability and gap conditions. The indistinguishability condition is satisfied by showing that the obfuscating oracle $f_\A$ with $f_\by$ and $f_\bz$ from the previous section satisfies the obfuscating condition. The main hardness result then immediately follows by Theorem~\ref{thm:meta}. In Section~\ref{sec:existence}, we show the existence of the  $\by_i$ blocks needed for the obfuscating oracle, which is the main technical part of this section. Then, we show the obfuscating condition in Section~\ref{sec:resistantC} and the gap condition in Section~\ref{sec:gapC}. We bound the subgradients of any function $f \in \F_r$ in Section~\ref{sec:subgradients}. Finally, we conclude with the main result in Section~\ref{sec:mainresult}.

\subsection{The existence of $\by_i$ for the obfuscating oracle}
\label{sec:existence}

We show the existence of the blocks $\by_i$ needed for the obfuscating oracle $f_\A$ defined in Section~\ref{sec:obfuscating}.

\begin{lemma}
\label{lem:existence} For any $i \in [r+1]$ and randomized collection of queries $\cup_{j=1}^i \Q_j$,  there exists $\by_i$ such that 
$$\Pr_{\bx \in \cup_{j=1}^i \Q_j}\left[|\bx^\intercal \by_i| < \epsilon, \text{ for all } \bx \in \cup_{j=1}^i \Q_j\right] = 1 - n^{-\omega(1)}.$$
\end{lemma}

The remainder of Section~\ref{sec:existence} is devoted to proving Lemma~\ref{lem:existence}. The main idea of the proof is that instead of considering all possible randomized collection of queries $\cup_{j=1}^i \Q_j$, we consider a random $ \by_i \sim \U$ where $\U$ is the uniform distribution over all binary vectors $\{-1,1\}^{n/(r+1)}$. The next lemma switches the randomization from the queries to the randomization of $\by_i \sim \U$. This is done by reducing the problem of showing the claim for any randomized collection of queries $\cup_{j=1}^i \Q_j$ to showing the claim for any $\bx$ and random $ \by_i \sim \U$. This is useful since standard concentrations bounds  apply more easily to $ \by_i \sim \U$ than to $\cup_{j=1}^i \Q_j$.

\begin{restatable}{rLem}{lemRandomization}
\label{lem:randomization}
 Assume that for all $\bx \in \mathcal{W}$, w.p. $1 - n^{-\omega(1)}$ over  $ \by_i \sim \U$, we have that
$|\bx^\intercal \by_i| < \epsilon$.  Then, for any (possibly randomized) collection of $\poly(n)$  queries $ \cup_{j=1}^i \Q_j$, there exists a deterministic  $\by_i$ such that with probability  $1 - n^{-\omega(1)}$ over the randomization of $\cup_{j=1}^i \Q_j$,
for all queries $\bx \in \cup_{j=1}^i \Q_j$, $|\bx^\intercal \by_i| < \epsilon$.
\end{restatable}
\begin{proof}
We denote by  $I(\by, \Q)$  the event that $|\bx^\intercal \by| < \epsilon$ for all $\bx \in \Q$. Let $\cup_{j=1}^i \Q_j$ be a  randomized collection of $\poly(n)$  queries and let $\by_i \sim \U$ be such that  for all $\bx$, w.p. $1 - n^{-\omega(1)}$ over the randomization of $\by_i \sim \U$,
 we have that
$|\bx^\intercal \by_i| < \epsilon$.

 Let $\Q$ be any realization of the randomized collection of  queries $\cup_{j=1}^i \Q_j$. By a union bound over the $\poly(n)$ queries $\bx \in \Q$, $\Pr_{\by_i \sim \U}\left[I(\by_i, \Q))\right] 
\geq 1 - n^{-\omega(1)}.$ We obtain
\begin{align*}
\max_{\by_i \in  \{-1,1\}^{n/(r+1)}}\Pr_{\Q}\left[I(\by_i,\Q)\right] \geq  \Pr_{\by_i \sim \U} \Pr_{\Q}\left[I( \by_i,\Q)\right]  
\geq 1 - n^{-\omega(1)}
\end{align*}
and thus  there exists some $\by_i \in  \{-1,1\}^{n/(r+1)}$ such that w.p. $1 - n^{-\omega(1)}$ over the randomization of $\cup_{j=1}^i \Q_j$,
for all queries $\bx \in \cup_{j=1}^i \Q_j$, $|\bx^\intercal \by| < \epsilon$.
\end{proof}

 Before showing the condition needed for Lemma~\ref{lem:randomization}, we state the following version of Hoeffding's inequality.

\begin{lemma}[Hoeffding's inequality]
\label{lem:hoeffding}
Let $X_1, \dots, X_n$ be independent random variables with values in $[a,b]$. Let $X = \sum_{i=1}^n X_i$ and $\mu = \E[X]$. Then for every $t > 0$,
$$\Pr\left[|S - \mu| \geq t\right] \leq 2 e^{- \frac{2t^2}{n (b-a)^2}}.$$
\end{lemma}

Next, we show the condition needed for Lemma~\ref{lem:randomization}, namely that for all $\bx \in \mathcal{W}$, w.p. $1 - n^{-\omega(1)}$ over  $ \by_i \sim \U$, 
$|\bx^\intercal \by_i|$ concentrates. This follows from a straightforward application of Hoeffding's inequality.

\begin{lemma}
\label{lem:concentration}
For all $\bx \in \mathcal{W}$, with probability $1 - n^{-\omega(1)}$ over  $ \by_i \sim \U$, we have that
$|\bx^\intercal \by_i| < \epsilon$
\end{lemma}
\begin{proof} Consider $\bx^{\intercal} \by_i$ with $ \by_i \sim \U$. By ignoring indices $j$ such that $y_j^i = 0$ and considering indices $j$ such that $y_j^i = \pm 1$ with probability $1/2$ each independently, $\bx^{\intercal} \by_i$  is the sum of $n/(r+1)$ independent random variables with values in $[-D/(2\sqrt{n}), D/(2\sqrt{n})]$ and expected value $0$, by Hoeffding's inequality (Lemma~\ref{lem:hoeffding}), we get
\begin{align*}
\Pr_{\by_i \sim \U} \left[\left|\bx^{\intercal} \by_i \right| < \epsilon\right] & =
\Pr_{\by_i \sim \U} \left[\left|\bx^{\intercal} \by_i \right| < \frac{D \log n }{2 \sqrt{r+1}}  \right] 
 \geq 1 - 2e^{-\frac{2 \left(\frac{1}{2} D \cdot  \sqrt{1/(r+1)} \cdot \log n\right)^2}{(n/(r+1))(D/\sqrt{n})^2}} \\
& \geq 1 - 2e^{-\frac{ \log^2 n}{2}}  
 = 1 - 2n^{-\frac{\log n}{2}} 
 = 1 - n^{- \omega(1)} \qedhere
\end{align*}
\end{proof}

Lemma~\ref{lem:existence} then follows immediately from Lemmas~\ref{lem:randomization} and \ref{lem:concentration}.

\begin{proof}[Proof of Lemma~\ref{lem:existence}] We combine
 Lemma~\ref{lem:randomization} and Lemma~\ref{lem:concentration}.
 \end{proof}

\subsection{The obfuscating condition}
\label{sec:resistantC}

We show that the obfuscating oracle $f_\A$ together with $f_\by$ and $f_\bz$ defined in Section~\ref{sec:obfuscating}  satisfy the obfuscating condition. By Lemma~\ref{lem:obfuscating}, this implies the indistinguishability condition for $f_\by$ and $f_\bz$.  The main idea to show the obfuscating condition for queries $\bx$ at round $i$ is that for any $j > i$,  $\bx^\intercal_j \by_j$ concentrates with high probability.

\begin{lemma} 
\label{lem:resistant}
Let $\A$ be an $r$-adaptive algorithm, then $f_\A$, $f_\by$, and $f_{\bz}$ satisfy the obfuscating condition: with probability $1 - n^{-\omega(1)}$, for all rounds $i \leq r$ and all queries $\bx$ by $\A$ at round $i$ to obfuscating oracle $f_\A$, 
 $$f_\A(\bx) = \gamma \cdot \max_{j \in [i]} \left(\bx^{\intercal}_j \by_j - 2 j \epsilon\right) = f_{\by}(\bx)     = f_{\bz}(\bx).$$
\end{lemma}
\begin{proof}
Consider  round $i$ of the algorithm querying the obfuscating oracle $f$. By definition of $\by_j$, for $j \geq i$, 
$$\Pr_{\bx \in \cup_{\ell=1}^j \Q_\ell}\left[\left|\bx^\intercal_j \by_j\right| < \epsilon, \text{ for all } \bx \in \cup_{\ell=1}^j \Q_\ell\right] = 1 - n^{-\omega(1)}.$$
In particular, this implies that for $j \geq i$, we have
$$\Pr_{\bx \in \Q_i}\left[\left|\bx^\intercal_j \by_j\right| < \epsilon, \text{ for all } \bx \in \Q_i\right] = 1 - n^{-\omega(1)}.$$
By a union bound, we get
$$\Pr_{\bx \in \Q_i}\left[\left|\bx^\intercal_j \by_j \right| < \epsilon, \text{ for all } \bx \in \Q_i \text{ and for all } j \geq i\right] = 1 - n^{-\omega(1)}.$$
Assume that $|\bx^\intercal_j \by_j| < \epsilon, \text{ for all } \bx \in \Q_i \text{ and for all } j \geq i$. This implies that
$$\bx^\intercal_i \by_i - 2i\epsilon > \bx^\intercal_j \by_j - 2j \epsilon$$
for all $ \bx \in \Q_i$ and for all $j > i$.  If $|\bx^\intercal_{r+1} \by_{r+1}| < \epsilon$ then it is also the case that $|\bx^\intercal_{r+1} (-\by_{r+1})| < \epsilon$. Thus, 
 $$\gamma \cdot \max_{\ell \in [i]} \left(\bx^{\intercal}_\ell \by_\ell - 2 \ell \epsilon\right) =  f_{\by}(\bx) =   f_{\bz}(\bx) .$$
\end{proof}

\subsection{The gap condition}
\label{sec:gapC}

We show that $f_\by$ and $f_\bz$  satisfy the $\alpha$-gap condition with $\alpha = GD\left(\frac{1}{2\sqrt{r+1}} -  \frac{(r+1/2) \log n}{\sqrt{n}} \right)$. The main observation for the gap condition is that for all $\bx$ and  $\by_{r+1}$, $$\max(\bx^{\intercal}_{r+1} \by_{r+1},  \bx^{\intercal}_{r+1} \bz_{r+1}) = \max(\bx^{\intercal}_{r+1} \by_{r+1}, - \bx^{\intercal}_{r+1} \by_{r+1}) \geq 0.$$ Thus, for all $\by$ there is no $\bx$ which is a good solution to both $f_{\by}(\bx)$ and  $ f_{\bz}(\bx)$.

\begin{lemma} 
\label{lem:gapC} Assume $\gamma = \sqrt{(r+1)/n} \cdot G$.
For any $\by \in \{-1,1\}^{n}$,
$\min_\bx f_{\by}(\bx) = \min_\bx f_{\bz}(\bx)$
 and for all $\bx \in \mathcal{W}$,
$$\max(f_{\by}(\bx), f_{\bz}(\bx)) - \min_\bx f_{\by}(\bx) \geq G D\left( \frac{1}{2\sqrt{r+1}} -  \frac{(r+1/2) \log n}{\sqrt{n}} \right).$$
\end{lemma}
\begin{proof}
A minimizer  for $f_{\by}$ is $\bx^{\star}$ such that  $x^{\star}_j =  - D / (2\sqrt{n})$ if $y_j = 1$ and $x^{\star}_j =   D / (2\sqrt{n})$ if $y_j = - 1$. With $\epsilon = (D \log n)/(2\sqrt{r+1})$,
$$f(\bx^{\star}) = \gamma \cdot \left((\bx^{\star}_{1})^{\intercal} \by_{1} - \frac{D  \log n}{2 \sqrt{r+1}}\right) =   \gamma \cdot \left(-\frac{\sqrt{n}D }{2(r+1)} - \frac{D  \log n}{2 \sqrt{r+1}}\right).$$
We construct a minimizer $\bx^{\star}$ for $f_{\bz}$ similarly  and get  $$\gamma \cdot \left(-\frac{\sqrt{n}D }{2(r+1)} - \frac{D  \log n}{2 \sqrt{r+1}}\right) =  \min_\bx f_{\by}(\bx) = \min_\bx f_{\bz}(\bx).$$
By the definition of $f_{\by}$, for all $\bx$, we have
\begin{align*}
\max(f_{\by}(\bx), f_{\bz}(\bx))  \geq \gamma \cdot \max(\bx^{\intercal}_{r+1} \by_{r+1}, - \bx^{\intercal}_{r+1} \by_{r+1})  - \gamma 2(r+1) \epsilon 
 \geq - \gamma 2(r+1) \epsilon .
\end{align*}

Thus, for all $\bx$, with $\epsilon = (D \log n)/(2\sqrt{r+1})$ and $\gamma = \sqrt{(r+1)/n} \cdot G$,
\begin{align*}
\max(f_{\by}(\bx), f_{\bz}(\bx)) - f_{\by}(\bx^{\star})
  & \geq  \gamma \left( - \sqrt{r+1} D  \log n +\frac{\sqrt{n}D }{2(r+1)} + \frac{D  \log n}{2 \sqrt{r+1}}\right) \\
 & =  G D\left( \frac{1}{2\sqrt{r+1}} -  \frac{(r+1/2) \log n}{\sqrt{n}} \right). \qedhere
\end{align*}
\end{proof}

\subsection{The subgradients of $\F_r$}
\label{sec:subgradients}

It remains to bound the subgradients of the functions in $\F_r$.  We use $\gamma = \sqrt{\frac{r+1}{n}} \cdot G$ and the following standard lemma for subdifferentials.

\begin{lemma}[\cite{nesterov2013introductory}, Lemma 3.1.10]
\label{lem:nesterov}
Let the function $f_i(x), i = 1, \ldots, m,$ be closed and convex. Then the function $f(\bx) = \max_{i \in [m]} f_i(\bx)$ is also closed and convex. For any $\bx \in \intt(\dom f) = \cap_{i\in [m]} \intt (\dom f_i)$ we have
$\partial f(\bx) = \conv \{\partial f_i(\bx) | i \in I(\bx)\}$
where $I(\bx) = \{i : f_i(\bx) = f(\bx)\}$.
\end{lemma}

We bound the norm of  subgradients $\bg$ of functions $f \in \F_r$ using the above lemma.
\begin{lemma}
\label{lem:subgradients}
Let $\bg \in \partial f(\bx)$ for any $f \in \F_r$ and any $\bx$.  If $\gamma = \sqrt{\frac{r+1}{n}} \cdot G$, then $\|\bg\|^2 \leq G^2$.
\end{lemma}
\begin{proof}
Let $\bg \in \partial f(\bx)$ for some $\bx$, then by Lemma~\ref{lem:nesterov}, $\bg \in \conv \{\gamma \by_i | i \in I(\bx)\}.$ Thus, $$\bg = \gamma \sum_{i \in [r]} \alpha_i \by_i$$ for some $\alpha_1, \ldots, \alpha_r \geq 0 $ such that $\sum_{i \in [r+1]} \alpha_i = 1$, and we get
$$\|\bg\|^2 = \sum_{i \in [r+1]} \sum_{j \in \left[(i-1)\frac{n}{r+1}+1: i\frac{n}{r+1}\right]} (\gamma \alpha_i)^2 =  \frac{n}{r+1} \gamma^2 \sum_{i \in [r+1]}  (\alpha_i)^2 \leq \frac{n}{r+1} \gamma^2 \sum_{i \in [r+1]}  \alpha_i \leq \frac{n}{r+1} \gamma^2 = G^2.$$
\end{proof}

\subsection{Main result}
\label{sec:mainresult}

We are now ready to show our main result.

\thmmaintwo*
\begin{proof}
Consider $\F_r$ from Section~\ref{sec:resistantconstruction}. By Lemma~\ref{lem:subgradients}, $\|\bg\|^2 \leq G^2$ for all subgradients $\bg$ of all $f \in \F_r$. By Lemma~\ref{lem:resistant} and \ref{lem:gapC}, $\F_r$ satisfies the indistinguishability condition  and $\alpha$-gap conditions  with $\alpha= G D\left( \frac{1}{2\sqrt{r+1}} -  \frac{(r+1/2) \log n}{\sqrt{n}} \right)$. Thus, by Theorem~\ref{thm:meta}, there is no $r$-adaptive algorithm that finds for all $f \in \F_r$, with probability $\omega(1/n)$ over the randomization of the algorithm, a solution $\bx_r$ s.t.
$f(\bx_r) - \min_{\bx \in \mathcal{W}}f(\bx) \leq GD \left(\frac{1}{2\sqrt{r+1}} - \frac{(r+1/2)\log n}{\sqrt{n}}\right)$
over domain $\mathcal{W} = \left[-\frac{D}{2\sqrt{n}}, \frac{D}{2\sqrt{n}}\right]^n$.
\end{proof}


\subsubsection*{Strongly convex functions}
\label{sec:stronglyconvex}

We extend the previous result to  $\lambda$-strongly convex functions. The hard family of functions $\F_r^\lambda$ for strongly convex functions is similar to $\F_r$, but is defined with an additional $\frac{\lambda}{2} \|\bx\|^2$ additive terms in the functions to obtain $\lambda$-strongly convex functions. Formally,
$$f^\lambda_{\by}(\bx) := \sqrt{\frac{r+1}{n}}  \cdot \frac{G}{2} \cdot  \max_{i \in [r+1]} \left(\bx^{\intercal}_i \by_i -  2 i  \epsilon\right) + \frac{\lambda}{2} \|\bx\|^2$$
The functions $f^\lambda_{\by}(\bx)$  are $\lambda$-strongly convex since $\frac{\lambda}{2} \|\bx\|^2$ is $\lambda$-strongly convex and  the sum of a convex function and a $\lambda$-strongly convex function is a $\lambda$-strongly convex function. Similarly as for $\F_r$, the family of functions is
$$\F^{\lambda}_r := \{f^{\lambda}_{\by} : \by \in \{-1,1\}^{n}\}.$$

We consider the identical construction of $\by_1, \ldots, \by_{r+1}, \bz_{r+1}$ for the obfuscating oracle $f_\A$, $f_\by$, and $f_\bz$ as for $\F_r$ to show the indistinguishability and gap properties. The $\alpha$-gap property holds for a different $\alpha$.

\begin{lemma} 
\label{lem:gapC2}
For any $\by \in \{-1,1\}^{n}$,
$\min_\bx f^\lambda_{\by}(\bx) = \min_\bx f^\lambda_{\bz}(\bx)$
 and for all $\bx \in \mathcal{W}$,
$$\max(f^\lambda_{\by}(\bx), f^\lambda_{\bz}(\bx)) - \min_\bx f^\lambda_{\by}(\bx) \geq   \frac{G^2}{\lambda}\left(\frac{1 }{8(r+1)} - \sqrt{\frac{r+ 1}{n}}\frac{\log n}{2}\right) .$$
\end{lemma}
\begin{proof}
By the definition of $f^\lambda_{\by}$, for all $\bx$, similarly as in Lemma~\ref{lem:gapC}, $$\max(f^\lambda_{\by}(\bx), f^\lambda_{\by,-\by_{r+1}}(\bx)) \geq  -  \sqrt{\frac{r+1}{n}}  G (r+1) \epsilon.$$ Let $\bx^{\star}$ be defined as in Lemma~\ref{lem:gapC}. Then,
$$f^\lambda_{\by,-\by_{r+1}}(\bx^{\star}) = \sqrt{\frac{r+1}{n}} \frac{G}{2} \left((\bx^{\star}_1)^{\intercal} \by_1 - \frac{D  \log n}{2 \sqrt{r+1}}\right) + \frac{\lambda}{2} \|\bx\|^2 \leq  - \sqrt{\frac{r+1}{n}}   \frac{G \sqrt{n}D }{4(r+1)}  + \frac{\lambda D^2}{8}$$
and
 $$ \min_\bx f^\lambda_{\by}(\bx) =  \min_\bx f^\lambda_{\bz}(\bx) \leq - \sqrt{\frac{r+1}{n}} G  \frac{\sqrt{n}D }{2(r+1)}  + \frac{\lambda D^2}{8}.$$
With $\epsilon = \frac{D \log n}{2 \sqrt{r+1}}$ and $D = \frac{G}{\lambda \sqrt{r+1}}$, we conclude that
\begin{align*}
 \max(f^\lambda_{\by}(\bx), f^\lambda_{\bz}(\bx)) - \min_\bx f^\lambda_{\by}(\bx)  & \geq  -  \sqrt{\frac{r+1}{n}}  G (r+1) \epsilon  + \sqrt{\frac{r+1}{n}}   \frac{G\sqrt{n}D }{4(r+1)} - \frac{\lambda D^2}{8}\\ 
 &=  \sqrt{\frac{r+1}{n}}  G \left(\frac{\sqrt{n}D }{4(r+1)}  -   (r+1) \epsilon \right) -    \frac{\lambda D^2}{8} \\
 &=  \sqrt{\frac{r+1}{n}}  G \left(\frac{\sqrt{n}D }{4(r+1)}  -\frac{D}{2} \log n \sqrt{r+1} \right) -    \frac{\lambda D^2}{8} \\
 &=  GD \left(\frac{1 }{4\sqrt{r+1}} - \frac{(r+1)\log n}{2\sqrt{n}} \right) -    \frac{\lambda D^2}{8} \\
 &=  \frac{G^2}{\lambda}\left(\frac{1 }{4(r+1)} - \frac{\sqrt{r+ 1}\log n}{2\sqrt{n}} \right) -    \frac{G^2}{8\lambda(r+1)} \\
  &=  \frac{G^2}{\lambda}\left(\frac{1 }{8(r+1)} - \sqrt{\frac{r+ 1}{n}}\frac{\log n}{2}\right). 
\end{align*}
\end{proof}

We bound the norm of  subgradients $\bg$ of functions $f \in \F^{\lambda}_r$.
\begin{lemma}
\label{lem:subgradientsStrongly}
Let $\bg \in \partial f(\bx)$ for any $f \in \F_r$ and any $\bx \in \mathcal{W}$.  If $D = \frac{G}{\lambda \sqrt{r+1}}$, then $\|\bg\|^2 \leq G^2$.
\end{lemma}
\begin{proof}
Let $\bg \in \partial f(\bx)$ for some $\bx$, then by Lemma~\ref{lem:nesterov}, 
 $$\bg = \sqrt{\frac{r+1}{n}}   \frac{G}{2}   \sum_{i \in [r]} \alpha_i \by_i + \lambda \bx$$ 
 for some $\alpha_1, \ldots, \alpha_r \geq 0 $ such that $\sum_{i \in [r+1]} \alpha_i = 1$. Thus,
\begin{align*}
\|\bg\|^2 & = \sum_{i \in [r+1]} \sum_{j \in \left[(i-1)\frac{n}{r+1}+1: i\frac{n}{r+1}\right]} (\sqrt{\frac{r+1}{n}}   \frac{G}{2} \alpha_i + \lambda x_j)^2 \\
& \leq   \sum_{i \in [r+1]} \sum_{j \in \left[(i-1)\frac{n}{r+1}+1: i\frac{n}{r+1}\right]} \left(\sqrt{\frac{r+1}{n}}   \frac{G}{2} \alpha_i + \lambda \frac{D}{2\sqrt{n}}\right)^2 \\
& =   \sum_{i \in [r+1]} \sum_{j \in \left[(i-1)\frac{n}{r+1}+1: i\frac{n}{r+1}\right]} \left(\frac{r+1}{n}  \frac{G^2}{4} \alpha_i^2 +  \sqrt{r+1}   \alpha_i \lambda \frac{GD}{2n}+  \frac{(\lambda D)^2}{4n}\right) \\
& \leq   \sum_{i \in [r+1]} \sum_{j \in \left[(i-1)\frac{n}{r+1}+1: i\frac{n}{r+1}\right]} \left(\frac{r+1}{n}   \frac{G^2}{4}  \alpha_i +  \sqrt{r+1}   \alpha_i \lambda \frac{GD}{2n}+ \frac{(\lambda D)^2}{4n}\right) \\
& =   \sum_{i \in [r+1]} \left( \frac{G^2}{4} \alpha_i +    \alpha_i \lambda \frac{GD}{2\sqrt{r+1}} + \frac{(\lambda D)^2}{4(r+1) }\right) \\
& =   \sum_{i \in [r+1]} \left( \frac{G^2}{4}  \alpha_i +    \alpha_i \lambda \frac{GD}{2\sqrt{r+1}} + \frac{(\lambda D)^2}{4(r+1) }\right) \\
& =    \frac{G^2}{4} +    \lambda \frac{GD}{2\sqrt{r+1}} + \frac{(\lambda D)^2}{4 } \\
& =    \frac{G^2}{4} +     \frac{G^2}{2(r+1)} +  \frac{G^2}{4(r+1) } \\
 & \leq G^2.
\end{align*}
\end{proof}

We obtain the following result for $\lambda$-strongly convex functions. 

\thmmainone*
\begin{proof}
Consider  the family of functions $\F^{\lambda}_r$. First, by Lemma~\ref{lem:subgradientsStrongly}, for  for any $f \in \F_r,$ $\bx \in \mathcal{W}$, and $\bg \in \partial f(\bx)$, we have   $\|\bg\|^2 \leq G^2$. Since $f^{\lambda}_{\by} = f_{\by} + \frac{\lambda}{2} \|\bx\|^2$ and $\frac{\lambda}{2} \|\bx\|^2$ is independent from $\by$, $\F^{\lambda}_r$ also satisfies the indistinguishability condition. By Lemma~\ref{lem:gapC}, $\F^{\lambda}_r$ satisfies the $\alpha$-gap conditions  with $\alpha=  \frac{G^2}{\lambda}\left(\frac{1 }{8(r+1)} - \sqrt{\frac{r+ 1}{n}}\frac{\log n}{2}\right)$. By  Theorem~\ref{thm:meta}, we get the desired result.
\end{proof}

\newpage
\bibliographystyle{alpha}
 \bibliography{biblio}

\newcommand{\etalchar}[1]{$^{#1}$}
\begin{thebibliography}{WWMS18}

\bibitem[AAAK17]{AAAK17}
Arpit Agarwal, Shivani Agarwal, Sepehr Assadi, and Sanjeev Khanna.
\newblock Learning with limited rounds of adaptivity: Coin tossing, multi-armed
  bandits, and ranking from pairwise comparisons.
\newblock In {\em COLT}, pages 39--75, 2017.

\bibitem[AWBR09]{agarwal2009information}
Alekh Agarwal, Martin~J Wainwright, Peter~L Bartlett, and Pradeep~K Ravikumar.
\newblock Information-theoretic lower bounds on the oracle complexity of convex
  optimization.
\newblock In {\em Advances in Neural Information Processing Systems}, pages
  1--9, 2009.

\bibitem[B{\etalchar{+}}15]{bubeck2015convex}
S{\'e}bastien Bubeck et~al.
\newblock Convex optimization: Algorithms and complexity.
\newblock {\em Foundations and Trends{\textregistered} in Machine Learning},
  8(3-4):231--357, 2015.

\bibitem[BBS18]{BBS18}
Eric Balkanski, Adam Breuer, and Yaron Singer.
\newblock Non-monotone submodular maximization in exponentially fewer
  iterations.
\newblock {\em CoRR}, abs/1807.11462, 2018.

\bibitem[BGP17]{braun2017lower}
G{\'a}bor Braun, Crist{\'o}bal Guzm{\'a}n, and Sebastian Pokutta.
\newblock Lower bounds on the oracle complexity of nonsmooth convex
  optimization via information theory.
\newblock {\em IEEE Transactions on Information Theory}, 63(7):4709--4724,
  2017.

\bibitem[BIS17]{BIS17}
Eric Balkanski, Nicole Immorlica, and Yaron Singer.
\newblock The importance of communities for learning to influence.
\newblock In {\em NIPS}, 2017.

\bibitem[BMW16]{BMW16}
Mark Braverman, Jieming Mao, and S~Matthew Weinberg.
\newblock Parallel algorithms for select and partition with noisy comparisons.
\newblock In {\em STOC}, pages 851--862, 2016.

\bibitem[BRS16]{BRS16}
Eric Balkanski, Aviad Rubinstein, and Yaron Singer.
\newblock The power of optimization from samples.
\newblock In {\em NIPS}, pages 4017--4025, 2016.

\bibitem[BRS17]{balkanski2017limitations}
Eric Balkanski, Aviad Rubinstein, and Yaron Singer.
\newblock The limitations of optimization from samples.
\newblock In {\em Proceedings of the 49th Annual ACM SIGACT Symposium on Theory
  of Computing}, pages 1016--1027. ACM, 2017.

\bibitem[BRS18]{BRS18}
Eric Balkanski, Aviad Rubinstein, and Yaron Singer.
\newblock An exponential speedup in parallel running time for submodular
  maximization without loss in approximation.
\newblock {\em arXiv preprint arXiv:1804.06355}, 2018.

\bibitem[BS17a]{BS17}
Eric Balkanski and Yaron Singer.
\newblock Minimizing a submodular function from samples.
\newblock In {\em Advances in Neural Information Processing Systems 30: Annual
  Conference on Neural Information Processing Systems 2017, 4-9 December 2017,
  Long Beach, CA, {USA}}, pages 814--822, 2017.

\bibitem[BS17b]{BS17a}
Eric Balkanski and Yaron Singer.
\newblock The sample complexity of optimizing a convex function.
\newblock In {\em Conference on Learning Theory}, pages 275--301, 2017.

\bibitem[BS18a]{BS18}
Eric Balkanski and Yaron Singer.
\newblock The adaptive complexity of maximizing a submodular function.
\newblock In {\em STOC}, 2018.

\bibitem[BS18b]{BS18b}
Eric Balkanski and Yaron Singer.
\newblock Approximation guarantees for adaptive sampling.
\newblock {\em ICML}, 2018.

\bibitem[CG17]{CG17}
Clement Canonne and Tom Gur.
\newblock An adaptivity hierarchy theorem for property testing.
\newblock {\em arXiv preprint arXiv:1702.05678}, 2017.

\bibitem[Col88]{Col88}
Richard Cole.
\newblock Parallel merge sort.
\newblock {\em SIAM Journal on Computing}, 17(4):770--785, 1988.

\bibitem[CQ18]{chekuri2018submodular}
Chandra Chekuri and Kent Quanrud.
\newblock Submodular function maximization in parallel via the multilinear
  relaxation.
\newblock {\em arXiv preprint arXiv:1807.08678}, 2018.

\bibitem[CST{\etalchar{+}}17]{chen2017settling}
Xi~Chen, Rocco~A Servedio, Li-Yang Tan, Erik Waingarten, and Jinyu Xie.
\newblock Settling the query complexity of non-adaptive junta testing.
\newblock {\em arXiv preprint arXiv:1704.06314}, 2017.

\bibitem[DBW12]{DBW12}
John~C. Duchi, Peter~L. Bartlett, and Martin~J. Wainwright.
\newblock Randomized smoothing for stochastic optimization.
\newblock {\em {SIAM} Journal on Optimization}, 22(2):674--701, 2012.

\bibitem[DGS84]{duris1984lower}
Pavol Duris, Zvi Galil, and Georg Schnitger.
\newblock Lower bounds on communication complexity.
\newblock In {\em STOC}, pages 81--91, 1984.

\bibitem[DGSX12]{DGSX12}
Ofer Dekel, Ran Gilad{-}Bachrach, Ohad Shamir, and Lin Xiao.
\newblock Optimal distributed online prediction using mini-batches.
\newblock {\em Journal of Machine Learning Research}, 13:165--202, 2012.

\bibitem[DRY18]{duchi2018minimax}
John Duchi, Feng Ruan, and Chulhee Yun.
\newblock Minimax bounds on stochastic batched convex optimization.
\newblock In {\em Conference On Learning Theory}, pages 3065--3162, 2018.

\bibitem[EN18]{EN18}
Alina Ene and Huy~L Nguyen.
\newblock Submodular maximization with nearly-optimal approximation and
  adaptivity in nearly-linear time.
\newblock {\em arXiv preprint arXiv:1804.05379}, 2018.

\bibitem[FMZ18]{fahrbach2018submodular}
Matthew Fahrbach, Vahab Mirrokni, and Morteza Zadimoghaddam.
\newblock Submodular maximization with optimal approximation, adaptivity and
  query complexity.
\newblock {\em arXiv preprint arXiv:1807.07889}, 2018.

\bibitem[FW78]{F78}
Steven Fortune and James Wyllie.
\newblock Parallelism in random access machines.
\newblock In {\em Proceedings of the Tenth Annual ACM Symposium on Theory of
  Computing}, STOC '78, pages 114--118, New York, NY, USA, 1978. ACM.

\bibitem[Gol78]{G78}
Leslie~M. Goldschlager.
\newblock A unified approach to models of synchronous parallel machines.
\newblock In {\em Proceedings of the Tenth Annual ACM Symposium on Theory of
  Computing}, STOC '78, pages 89--94, New York, NY, USA, 1978. ACM.

\bibitem[HBCN09]{haupt2009compressive}
Jarvis~D Haupt, Richard~G Baraniuk, Rui~M Castro, and Robert~D Nowak.
\newblock Compressive distilled sensing: Sparse recovery using adaptivity in
  compressive measurements.
\newblock In {\em Signals, Systems and Computers, 2009 Conference Record of the
  Forty-Third Asilomar Conference on}, pages 1551--1555. IEEE, 2009.

\bibitem[HNC09]{HNC09}
Jarvis Haupt, Robert Nowak, and Rui Castro.
\newblock Adaptive sensing for sparse signal recovery.
\newblock In {\em Digital Signal Processing Workshop and 5th IEEE Signal
  Processing Education Workshop}, pages 702--707. IEEE, 2009.

\bibitem[IPW11]{IPW11}
Piotr Indyk, Eric Price, and David~P Woodruff.
\newblock On the power of adaptivity in sparse recovery.
\newblock In {\em FOCS}, pages 285--294. IEEE, 2011.

\bibitem[Nem94]{nemirovski1994parallel}
Arkadi Nemirovski.
\newblock On parallel complexity of nonsmooth convex optimization.
\newblock {\em Journal of Complexity}, 10(4):451--463, 1994.

\bibitem[Nes13]{nesterov2013introductory}
Yurii Nesterov.
\newblock {\em Introductory lectures on convex optimization: A basic course},
  volume~87.
\newblock Springer Science \& Business Media, 2013.

\bibitem[NW91]{nisan1991rounds}
Noam Nisan and Avi Widgerson.
\newblock Rounds in communication complexity revisited.
\newblock In {\em STOC}, pages 419--429, 1991.

\bibitem[NY83]{nemirovskii1983problem}
Arkadii Nemirovski and David~Borisovich Yudin.
\newblock Problem complexity and method efficiency in optimization.
\newblock 1983.

\bibitem[PRC{\etalchar{+}}16]{perchet2016batched}
Vianney Perchet, Philippe Rigollet, Sylvain Chassang, Erik Snowberg, et~al.
\newblock Batched bandit problems.
\newblock {\em The Annals of Statistics}, 44(2):660--681, 2016.

\bibitem[PS84]{papadimitriou1984communication}
Christos~H Papadimitriou and Michael Sipser.
\newblock Communication complexity.
\newblock {\em Journal of Computer and System Sciences}, 28(2):260--269, 1984.

\bibitem[RBGS18]{RBGS18}
Nir Rosenfeld, Eric Balkanski, Amir Globerson, and Yaron Singer.
\newblock Learning to optimize combinatorial functions.
\newblock In {\em Proceedings of the 35th International Conference on Machine
  Learning, {ICML} 2018, Stockholmsm{\"{a}}ssan, Stockholm, Sweden, July 10-15,
  2018}, pages 4371--4380, 2018.

\bibitem[RR09]{raginsky2009information}
Maxim Raginsky and Alexander Rakhlin.
\newblock Information complexity of black-box convex optimization: A new look
  via feedback information theory.
\newblock In {\em Communication, Control, and Computing, 2009. Allerton 2009.
  47th Annual Allerton Conference on}, pages 803--510. IEEE, 2009.

\bibitem[RR11]{raginsky2011information}
Maxim Raginsky and Alexander Rakhlin.
\newblock Information-based complexity, feedback and dynamics in convex
  programming.
\newblock {\em IEEE Transactions on Information Theory}, 57(10):7036--7056,
  2011.

\bibitem[RRWN11]{RRWN11}
Benjamin Recht, Christopher R{\'{e}}, Stephen~J. Wright, and Feng Niu.
\newblock Hogwild: {A} lock-free approach to parallelizing stochastic gradient
  descent.
\newblock In {\em Advances in Neural Information Processing Systems 24: 25th
  Annual Conference on Neural Information Processing Systems 2011. Proceedings
  of a meeting held 12-14 December 2011, Granada, Spain.}, pages 693--701,
  2011.

\bibitem[SS79]{S79}
Walter~J. Savitch and Michael~J. Stimson.
\newblock Time bounded random access machines with parallel processing.
\newblock {\em J. ACM}, 26(1):103--118, January 1979.

\bibitem[STU17]{smith2017interaction}
Adam Smith, Abhradeep Thakurta, and Jalaj Upadhyay.
\newblock Is interaction necessary for distributed private learning?
\newblock In {\em Security and Privacy (SP), 2017 IEEE Symposium on}, pages
  58--77. IEEE, 2017.

\bibitem[SZ13]{shamir2013stochastic}
Ohad Shamir and Tong Zhang.
\newblock Stochastic gradient descent for non-smooth optimization: Convergence
  results and optimal averaging schemes.
\newblock In {\em International Conference on Machine Learning}, pages 71--79,
  2013.

\bibitem[Val75]{Val75}
Leslie~G Valiant.
\newblock Parallelism in comparison problems.
\newblock {\em SIAM Journal on Computing}, 4(3):348--355, 1975.

\bibitem[WS16]{woodworth2016tight}
Blake~E Woodworth and Nati Srebro.
\newblock Tight complexity bounds for optimizing composite objectives.
\newblock In {\em Advances in neural information processing systems}, pages
  3639--3647, 2016.

\bibitem[WWMS18]{woodworth2018graph}
Blake Woodworth, Jialei Wang, Brendan McMahan, and Nathan Srebro.
\newblock Graph oracle models, lower bounds, and gaps for parallel stochastic
  optimization.
\newblock {\em arXiv preprint arXiv:1805.10222}, 2018.

\end{thebibliography}

\end{document}